\def\year{2018}\relax
\documentclass[letterpaper]{article} 
\usepackage{aaai18}  
\usepackage{times}  
\usepackage{helvet}  
\usepackage{courier}  
\usepackage{url}  
\usepackage{graphicx}  

\usepackage{amsmath, amsfonts, bbm, amsthm}
\usepackage{algorithm, algorithmicx, algpseudocode, nicefrac}
\usepackage{caption, subcaption}
\usepackage{appendix}

\newtheorem{theorem}{Theorem}

\newtheorem{definition}{Definition}
\newtheorem{corollary}{Corollary}

\newtheorem{proposition}{Proposition}

\newcommand{\eg}{\emph{e.g.}}
\newcommand{\ie}{\emph{i.e.}}
\newcommand\Tstrut{\rule{0pt}{2.6ex}}

\begin{document}
%
\title{Boosted Generative Models}
\author{Aditya Grover, Stefano Ermon\\
Computer Science Department\\
Stanford University\\
\texttt{\{adityag, ermon\}@cs.stanford.edu}\\
}
\maketitle
\begin{abstract}
We propose a novel approach for using unsupervised boosting to create an ensemble of generative models, where models are trained in sequence to correct earlier mistakes. Our meta-algorithmic framework can leverage any existing base learner that permits likelihood evaluation, including recent deep expressive models. Further, our approach allows the ensemble to include discriminative models trained to distinguish real data from model-generated data. We show theoretical conditions under which incorporating a new model in the ensemble will improve the fit and empirically demonstrate the effectiveness of our black-box boosting algorithms on density estimation, classification, and sample generation on benchmark datasets for a wide range of generative models.
\end{abstract}

\section{Introduction}\label{sec:intro}
 A variety of deep generative models have shown promising results on tasks spanning computer vision, speech recognition, natural language processing, and imitation learning~\cite{poon2011sum,oord2016pixel,kingma-iclr2014,goodfellow2014generative,zhao2017learning,li2017inferring}.
These parametric models differ from each other in their ability to perform various forms of tractable inference, learning algorithms, and objectives.
Despite significant progress, existing generative models cannot fit complex distributions with a sufficiently high degree of accuracy, limiting their applicability and leaving room for 
improvement.

In this paper, we propose a technique for ensembling (imperfect) generative models to improve their overall performance.
Our meta-algorithm is inspired by boosting, a technique used in supervised learning to
combine weak classifiers (\eg, decision stumps or trees), which individually might not perform well on a given classification task, into a more powerful ensemble.
The boosting algorithm will attempt to learn a classifier to correct for the mistakes made by reweighting the original dataset, and repeat this procedure recursively. Under some conditions on the weak classifiers' effectiveness, this procedure can drive the (training) error to zero~\cite{freund1999short}. Boosting can also be thought as a feature learning algorithm, where at each round a new feature is learned by training a classifier on a reweighted version of the original dataset. 
In practice, algorithms based on boosting perform extremely well in machine learning competitions~\cite{caruana2006empirical}.

We show that a similar procedure can be applied to generative models. Given an initial generative model that provides an imperfect fit to the data distribution, we construct a second model to correct for the error, and repeat recursively. The second model is also a generative one, which is trained on a reweighted version of the original training set. Our meta-algorithm is general and can construct ensembles of any existing generative model that permits (approximate) likelihood evaluation
such as fully-observed belief networks, sum-product networks, and variational autoencoders.
Interestingly, our method can also leverage powerful discriminative models. 
Specifically, we train a binary classifier to distinguish true data samples from ``fake'' ones generated by the current model and provide a principled way to include this discriminator in the ensemble. 

A prior attempt at boosting density estimation proposed a \textit{sum-of-experts} formulation~\cite{rosset2002boosting}. This approach is similar to supervised boosting where at every round of boosting we derive a reweighted additive estimate of the boosted model density. 
In contrast, our proposed framework uses multiplicative boosting which multiplies the ensemble model densities and can be interpreted as a \textit{product-of-experts} formulation. 
We provide a holistic theoretical and algorithmic framework for multiplicative boosting contrasting with competing additive approaches. Unlike prior use cases of product-of-experts formulations, our approach is \textit{black-box}, and we empirically test the proposed algorithms on several generative models from simple ones such as mixture models to expressive parameteric models such as sum-product networks and variational 
autoencoders.

Overall, this paper makes the following contributions:
\begin{enumerate}
\item We provide theoretical conditions for additive and multiplicative boosting under which incorporating a new model is guaranteed to improve the ensemble fit.
\item We design and analyze a flexible meta-algorithmic boosting framework for including both generative and discriminative models in the ensemble.
\item We demonstrate the empirical effectiveness of our algorithms for density estimation, generative classification, and sample generation on several benchmark datasets. 
\end{enumerate}

\section{Unsupervised boosting}\label{sec:theory}
Supervised boosting provides an algorithmic formalization of the hypothesis that a sequence of weak learners can create a single strong learner~\cite{schapire2012boosting}. 
Here, we propose a framework that extends boosting to unsupervised settings for learning generative models. For ease of presentation, all distributions are with respect to any arbitrary  $\mathbf{x} \in \mathbb{R}^d$, unless otherwise specified. We use upper-case symbols to denote probability distributions and assume they all admit absolutely continuous densities (denoted by the corresponding lower-case notation) on a reference measure $\mathrm{d}\mathbf{x}$. Our analysis naturally extends to discrete distributions, which we skip for brevity. 

Formally, we consider the following maximum likelihood estimation (MLE) setting. Given some  data points $X=\{\mathbf{x}_i \in \mathbb{R}^d\}_{i=1}^{m}$ sampled i.i.d. from an unknown distribution $P$, we provide a model class $\mathcal{Q}$ parameterizing the distributions that can be represented by the generative model and minimize the Kullback-Liebler (KL) divergence with respect to the true distribution:
\begin{align}\label{eq:MLE_objective}
\min_{Q \in \mathcal{Q}} D_{KL}(P \Vert Q).
\end{align}

In practice, we only observe samples from $P$ and hence, maximize the log-likelihood of the observed data $X$. Selecting the model class for maximum likelihood learning is non-trivial; MLE w.r.t. a small class can be far from $P$, whereas a large class poses the risk of overfitting in the absence of sufficient data, or even underfitting due to difficulty in optimizing non-convex objectives that frequently arise due to the use of latent variable models, neural networks, etc.

The boosting intuition is to greedily increase model capacity by learning a sequence of weak intermediate models $\{h_t \in \mathcal{H}_t\}_{t=0}^T$ that can correct for mistakes made by previous models in the ensemble. Here, $\mathcal{H}_t$ is a predefined model class (such as $\mathcal{Q}$) for $h_t$. We defer the algorithms pertaining to the learning of such intermediate models to the next section, and first discuss two mechanisms for deriving the final estimate $q_T$ from the individual density estimates at each round, $\{h_t\}_{t=0}^T$.

\subsection{Additive boosting}
In additive boosting, the final density estimate is an arithmetic average of the intermediate models:
\begin{align*}
q_T = \sum_{t=0}^T \alpha_t \cdot h_t
\end{align*}
where $0 \leq \alpha_t\leq 1$ denote the weights assigned to the intermediate models. The weights are re-normalized at every round to sum to 1 which gives us a valid probability density estimate. Starting with a base model $h_0$, we can express the density estimate after a round of boosting recursively as:
\begin{align*}
q_t = (1-\hat{\alpha}_t) \cdot q_{t-1} + \hat{\alpha}_t \cdot h_t
\end{align*}
where $\hat{\alpha}_t$ denotes the normalized weight for $h_t$ at round $t$. We now derive conditions on the intermediate models that guarantee ``progress'' in every round of boosting.

\begin{theorem}\label{thm:add_KL_red}
Let \small${\delta^t_{KL}(h_t, \hat{\alpha}_t) = D_{KL}(P \Vert Q_{t-1}) - D_{KL}(P \Vert Q_t)}$\normalsize{} denote the reduction in KL-divergence at the $t^{th}$ round of additive boosting. The following conditions hold:
 \begin{enumerate}
 \item Sufficient: If $\mathbb{E}_P \left[\log \frac{h_t}{q_{t-1}}\right] \geq 0$, then $\delta^t_{KL}(h_t, \hat{\alpha}_t) \geq 0$ for all
 $\hat{\alpha}_t \in [0,1]$.
 \item Necessary: If $\exists \hat{\alpha}_t \in (0, 1]$ such that $\delta^t_{KL}(h_t, \hat{\alpha}_t) \geq 0$, then $\mathbb{E}_P \left[\frac{h_t}{q_{t-1}}\right] \geq 1$.
 \end{enumerate}
\end{theorem}
\begin{proof}
In Appendix~\ref{proof:add_KL_red}.
\end{proof}

The sufficient and necessary conditions require that the expected log-likelihood and likelihood respectively of the current intermediate model, $h_t$ are better-or-equal than those of the combined previous model, $q_{t-1}$ under the true distribution when compared using density ratios. 
Next, we consider an alternative formulation of multiplicative boosting for improving the model fit to an arbitrary data distribution.

\subsection{Multiplicative boosting}

In multiplicative boosting, we factorize the final density estimate as a geometric average of  $T+1$ intermediate models $\{h_t\}_{t=0}^T$, each assigned an exponentiated weight $\alpha_t$:
\begin{align*}
q_T = \frac{\prod_{t=0}^T h_t^{\alpha_t}}{Z_T}
\end{align*}
where the partition function $Z_T = \int \prod_{t=0}^T h_t^{\alpha_t} \,\mathrm{d}\mathbf{x}$. 
Recursively, we can specify the density estimate as:
\begin{align}\label{eq:q_t}
\tilde{q}_t &= h_t^{\alpha_t} \cdot \tilde{q}_{t-1}
\end{align}
where $\tilde{q}_t$ is the unnormalized estimate at round $t$. The base model $h_0$ is learned using MLE. The conditions on the intermediate models for reducing KL-divergence at every round are stated below.

\begin{theorem}\label{thm:KL_red}
Let \small${\delta^t_{KL}(h_t, \alpha_t) = D_{KL}(P \Vert Q_{t-1}) - D_{KL}(P \Vert Q_t)}$\normalsize{} denote the reduction in KL-divergence at the $t^{th}$ round of multiplicative boosting. The following conditions hold:
 \begin{enumerate}
 \item Sufficient: If $\mathbb{E}_P [\log h_t] \geq \log \mathbb{E}_{Q_{t-1}}[h_t]$, then $\delta^t_{KL}(h_t, \alpha_t) \geq 0$ for all
 $\alpha_t \in [0, 1]$.
 \item Necessary: If $\exists \alpha_t \in (0, 1]$ such that $\delta^t_{KL}(h_t, \alpha_t) \geq 0$, then $\mathbb{E}_P [\log h_t] \geq \mathbb{E}_{Q_{t-1}}[\log h_t]$.
 \end{enumerate}
\end{theorem}
\begin{proof}
In Appendix~\ref{proof:KL_red}.
\end{proof}

In contrast to additive boosting, the conditions above compare expectations under the true distribution with expectations under the {\em model distribution} in the previous round, $Q_{t-1}$. The equality in the conditions holds for $\alpha_t=0$, which corresponds to the trivial case where the current intermediate model is ignored in Eq.~\eqref{eq:q_t}. For other valid $\alpha_t$, the non-degenerate version of the sufficient inequality guarantees progress towards the true data distribution. Note that the intermediate models increase the overall capacity of the ensemble at every round. As we shall demonstrate later, we find models fit using multiplicative boosting to outperform their additive counterparts empirically suggesting the conditions in Theorem~\ref{thm:KL_red} are easier to fulfill in practice.

From the necessary condition, we see that a ``good" intermediate model $h_t$  assigns a better-or-equal log-likelihood under the true distribution as opposed to the model distribution, $Q_{t-1}$. This condition suggests two learning algorithms for intermediate models which we discuss next.

\section{Boosted generative models}\label{sec:algo}
In this section, we design and analyze meta-algorithms for multiplicative boosting of generative models. Given any base model which permits (approximate) likelihood evaluation, we provide a mechanism for boosting this model using an ensemble of generative and/or discriminative models.

\begin{algorithm}[t]
   \caption{GenBGM($X = \{\mathbf{x}_i\}_{i=1}^m, T$ rounds)}
   \label{alg:genbgm}
\begin{algorithmic}
\State Initialize $d_0(\mathbf{x}_i) = \nicefrac{1}{m}$ for all $ i = {1,2, \ldots, m}$.
\State Obtain base generative model $h_0$.
\State Set (unnormalized) density estimate  $\tilde{q}_0 = h_0$
   \For{$t = 1, 2,\ldots, T$}
   \State - Choose $\beta_t$ and update $d_{t}$ using Eq.~\eqref{eq:reweight_distribution}.
   \State - Train generative model $h_t$ to maximize Eq.~\eqref{eq:genbgm_obj}.
	\State - Choose $\alpha_t$.
	\State - Set density estimate $\tilde{q}_t = h_t^{\alpha_t} \cdot \tilde{q}_{t-1}$.
   \EndFor 
   \State Estimate $Z_T=\int \tilde{q}_T \;\mathrm{d}\mathbf{x} $. \\
   \Return $q_T = \tilde{q}_T/Z_T$.
\end{algorithmic}
\end{algorithm}

\subsection{Generative boosting}
Supervised boosting algorithms such as AdaBoost typically involve a reweighting procedure for training weak learners~\cite{freund1995desicion}. We can similarly train an ensemble of generative models for unsupervised boosting, where every subsequent model performs MLE w.r.t a reweighted data distribution $D_t$:
\begin{align}
\max_{h_t}\mathbb{E}_{D_t}[\log h_t] \label{eq:genbgm_obj}\\
\text{where }d_t \propto \left(\frac{p}{q_{t-1}}\right)^{\beta_t} \label{eq:reweight_distribution}
\end{align}
and $\beta_t \in [0, 1]$ is the reweighting coefficient at round $t$. Note that these coefficients are in general different from the model weights $\alpha_t$ that appear in  Eq.~\eqref{eq:q_t}. 

\begin{proposition}\label{thm:genbgm_reweight}
If we can maximize the objective in Eq.~\eqref{eq:genbgm_obj} optimally, then $\delta_{KL}^t(h_t, \alpha_t) \geq 0$ for any $\beta_t \in [0, 1]$
with the equality holding for $\beta_t=0$.
\end{proposition}
\begin{proof}
In Appendix~\ref{proof:genbgm_reweight}.
\end{proof}

While the objective in Eq.~\eqref{eq:genbgm_obj} can be hard to optimize in practice, the target distribution becomes easier to approximate as we reduce the reweighting coefficient. For the extreme case of $\beta_t=0$, the reweighted data distribution is simply uniform. There is no free lunch however, since a low $\beta_t$ results in a slower reduction in KL-divergence leading to a computational-statistical trade-off. 

The pseudocode for the corresponding boosting meta-algorithm, referred to as GenBGM, is given in Algorithm~\ref{alg:genbgm}. In practice, we only observe samples from the true data distribution, and hence, approximate $p$ based on the empirical data distribution which is defined to be uniform over the dataset $X$. At every subsequent round, GenBGM learns an intermediate model that maximizes the log-likelihood of data sampled from a reweighted data distribution. 

\begin{algorithm}[t]
   \caption{DiscBGM($X = \{\mathbf{x}_i\}_{i=1}^m, T$ rounds, $f$-div)}
   \label{alg:discbgm}
\begin{algorithmic}
\State Initialize $d_0(\mathbf{x}_i) = \nicefrac{1}{m}$ for all $ i = {1,2, \ldots, m}$.
\State Obtain base generative model $h_0$.
\State Set (unnormalized) density estimate  $\tilde{q}_0 = h_0$
   \For{$t = 1, 2, \ldots, T$}
   \State - Generate negative samples from $q_{t-1}$
   \State - Optimize $r_t$ to maximize RHS in
   Eq.~\eqref{eq:f_disc_obj}.
   \State - Set $h_t = \left[f'\right]^{-1} (r_t)$.
  \State - Choose $\alpha_t$.
   \State - Set density estimate $\tilde{q}_t =  h_t^{\alpha_t} \cdot \tilde{q}_{t-1}$.
   \EndFor 
   \State Estimate $Z_T=\int \tilde{q}_T\; \mathrm{d}\mathbf{x} $. \\
   \Return $q_T = \tilde{q}_T/Z_T$.
\end{algorithmic}
\end{algorithm}

\subsection{Discriminative boosting}

A base generative model can be boosted using a discriminative approach as well. Here, the intermediate model is specified as the density ratio obtained from a binary classifier. Consider the following setup: we observe an equal number of samples drawn i.i.d. from the true data distribution (w.l.o.g. assigned the label $y=+1$) and the model distribution in the previous round $Q_{t-1}$ (label $y=0$). 

\begin{definition}
Let $f: \mathbb{R}^{+}\rightarrow \mathbb{R}$ be any convex, lower semi-continuous function satisfying $f(1) = 0$. 
The $f$-divergence between $P$ and $Q$ is defined as, $D_f(P \Vert Q) = \int q \cdot f\left(\nicefrac{p}{q}\right) \mathrm{d} \mathbf{x}$.

\end{definition}
Notable examples include the Kullback-Liebler (KL) divergence, Hellinger distance, and the Jenson-Shannon (JS) divergence among many others. The binary classifier in discriminative boosting maximizes a variational lower bound on any $f$-divergence at round $t$:
\begin{align}\label{eq:f_disc_obj}
D_f\left(P \Vert Q_{t-1}\right) \geq \sup_{r_t\in \mathcal{R}_t} \left (\mathbb{E}_P[r_t] - \mathbb{E}_{Q_{t-1}}[f^\star(r_t)]\right). 
\end{align}
 where $f^\star$ denotes the Fenchel conjugate of $f$ and $r_t:\mathbb{R}^d \rightarrow \mathrm{dom}_{f^\star}$ parameterizes the classifier. Under mild conditions on $f$~\cite{nguyen2010estimating}, the lower bound in Eq.~\eqref{eq:f_disc_obj} is tight if
$r_t^\star = f'\left( \nicefrac{p}{q_{t-1}}\right)$.
 
Hence, a solution to Eq.~\eqref{eq:f_disc_obj} can be used to estimate density ratios. The density ratios naturally fit into the multiplicative boosting framework and provide a justification for the use of objectives of the form Eq.~\eqref{eq:f_disc_obj} for learning intermediate models as formalized in the proposition below.
\begin{proposition}\label{thm:f_optimal}
For any given $f$-divergence, let $r_t^\star$ denote the optimal solution to Eq.~\eqref{eq:f_disc_obj} in the $t^{th}$ round of boosting. Then, the model density at the end of the boosting round matches the true density if we set $\alpha_t=1$ and 
$h_t = \left[f'\right]^{-1} (r_t^\star)$
where $\left[f'\right]^{-1}$ denotes the inverse of the derivative of $f$.
\end{proposition}
\begin{proof}
In Appendix~\ref{proof:f_optimal}.
\end{proof}

The pseudocode for the corresponding meta-algorithm, DiscBGM is given in Algorithm~\ref{alg:discbgm}. At every round, we train a binary classifier to optimize the objective in Eq.~\eqref{eq:f_disc_obj} for a chosen $f$-divergence.  
As a special case, the negative of the cross-entropy loss commonly used for binary classification is also a lower bound on an $f$-divergence. While Algorithm~\ref{alg:discbgm} is applicable for any $f$-divergence, we will focus on cross-entropy henceforth to streamline the discussion.

\begin{corollary}\label{thm:bayes_optimal}
Consider the (negative) cross-entropy objective maximized by a binary classifier:
\begin{align}\label{eq:disc_obj}
   \sup_{c_t\in \mathcal{C}_t} \mathbb{E}_{P}[\log c_t] + \mathbb{E}_{Q_{t-1}}[\log(1-c_t)]. 
\end{align}
If a binary classifier $c_t$ trained to optimize Eq.~\eqref{eq:disc_obj} is Bayes optimal, then the model density after round $t$ matches the true density if we set $\alpha_t=1$ and 
$h_t= \nicefrac{c_t}{1-c_t}$.
\end{corollary}
\begin{proof}
In Appendix~\ref{proof:bayes_optimal}.
\end{proof}

In practice, a classifier with limited capacity trained on a finite dataset will not generally be Bayes optimal. The above corollary, however, suggests that a good classifier can provide a `direction of improvement', in a similar spirit to gradient boosting for supervised learning~\cite{freund1995desicion}. Additionally, if the intermediate model distribution $h_t$ obtained using the above corollary satisfies the conditions in Theorem~\ref{thm:KL_red}, it is guaranteed to improve the fit. 

The weights $\alpha_t\in [0,1]$ can be interpreted as our confidence in the classification estimates, akin to the step size used in gradient descent. 
While in practice we heuristically assign weights to the intermediate models, the greedy optimum value of these weights at every round is a critical point for $\delta^t_{KL}$ (defined in Theorem~\ref{thm:KL_red}). For example, in the extreme case where $c_t$ is uninformative, \ie, $c_t \equiv 0.5$, then $\delta^t_{KL}(h_t, \alpha_t)=0$ for all $\alpha_t\in [0,1] $. If $c_t$ is Bayes optimal, then $\delta^t_{KL}$ attains a maxima when $\alpha_t=1$ (Corollary~\ref{thm:bayes_optimal}).

\begin{figure}[t]
\centering
\begin{subfigure}[b]{0.45\columnwidth}
\includegraphics[width=\textwidth]{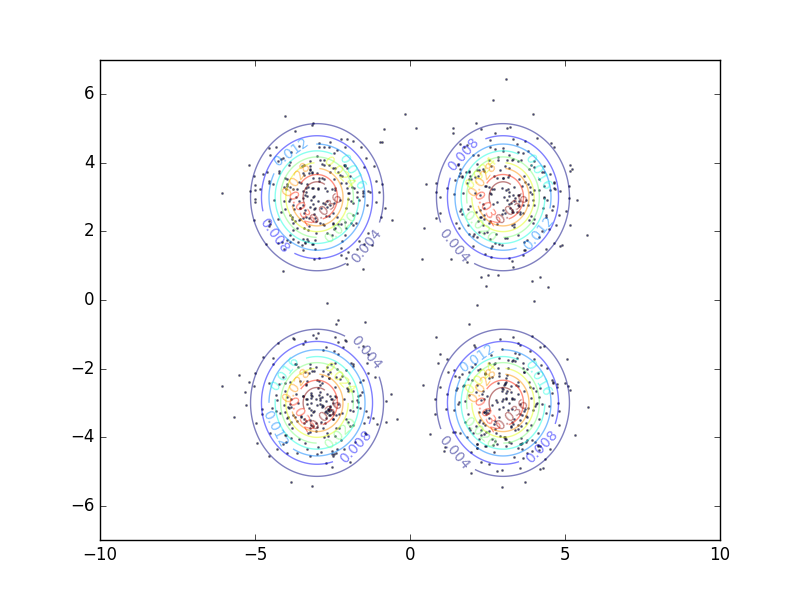}
\caption{Target}\label{fig:mog_setup_target}
\end{subfigure}
\begin{subfigure}[b]{0.45\columnwidth}
\includegraphics[width=\textwidth]{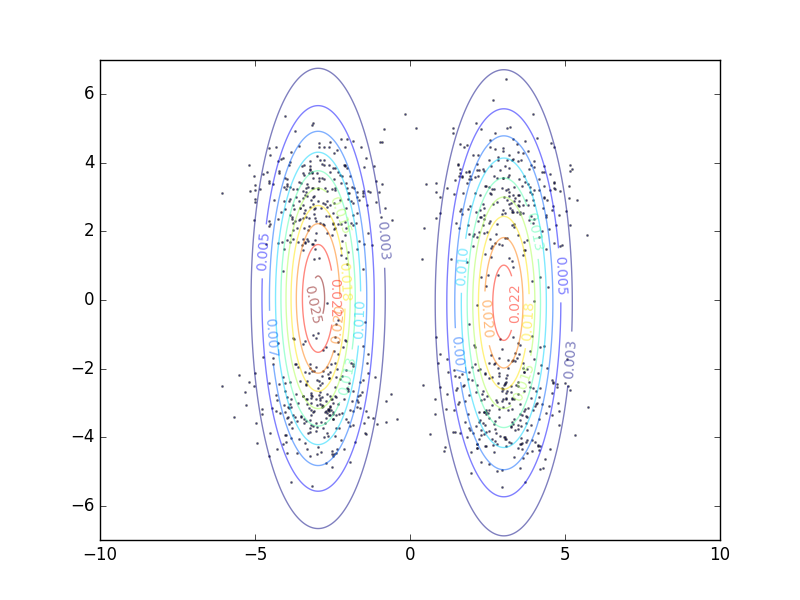}
\caption{Base model}\label{fig:mog_setup_base}
\end{subfigure}
\caption{The mixture of Gaussians setup showing (a) true density and (b) base (misspecified) model.}\label{fig:mog_setup}
\end{figure}

\subsection{Hybrid boosting}
Intermediate models need not be exclusively generators or discriminators; we can design a boosting ensemble with any combination of generators and discriminators.
If an intermediate model is chosen to be a generator, we learn a generative model using MLE after appropriately reweighting the data points. If a discriminator is used to implicitly specify an intermediate model, we set up a binary classification problem.

\subsection{Regularization}
In practice, we want boosted generative models (BGM) to generalize to data outside the training set $X$. Regularization in BGMs is imposed primarily in two ways. First, every intermediate model can be independently regularized by incorporating explicit terms in the learning objective, early stopping based on validation error, heuristics such as dropout, etc. 
Moreover, restricting the number of rounds of boosting is another effective mechanism for regularizing BGMs. Fewer rounds of boosting are required if the intermediate models are sufficiently expressive.

\section{Empirical evaluation}\label{sec:exp}
Our experiments are designed to demonstrate the superiority of the proposed boosting meta-algorithms on a wide variety of generative models and tasks. 
A reference implementation of the boosting meta-algorithms is available at \texttt{https://github.com/ermongroup/bgm}. Additional implementation details for the experiments below are given in Appendix~\ref{app:exp}.

\begin{figure*}[ht]
\centering
\begin{subfigure}[b]{0.22\textwidth}
\includegraphics[width=\textwidth]{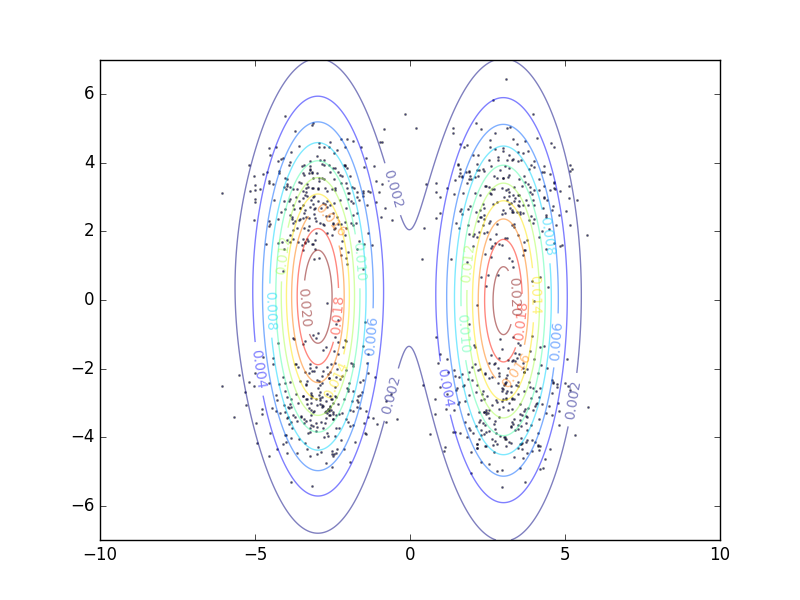}
\caption{Add model $(1)$}
\end{subfigure}
~
\begin{subfigure}[b]{0.22\textwidth}
\includegraphics[width=\textwidth]{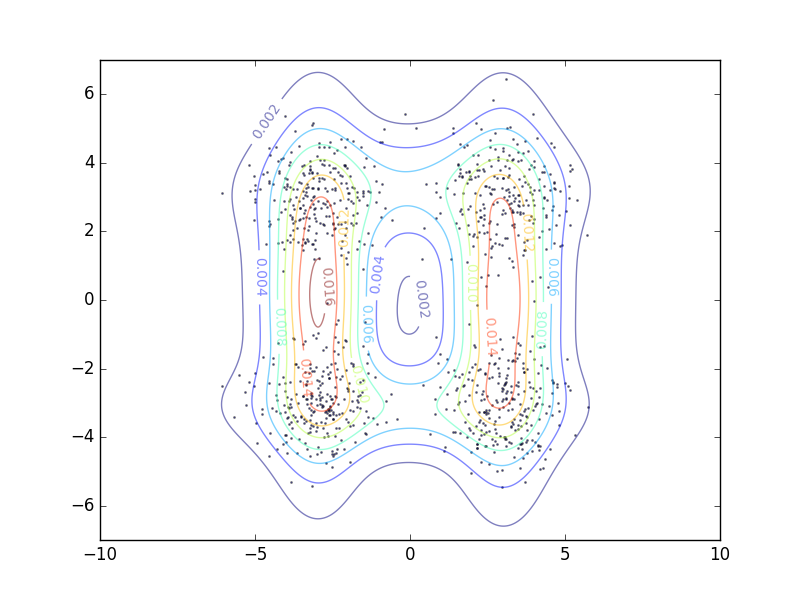}
\caption{Add model $(2)$}
\end{subfigure}
~
\begin{subfigure}[b]{0.22\textwidth}
\includegraphics[width=\textwidth]{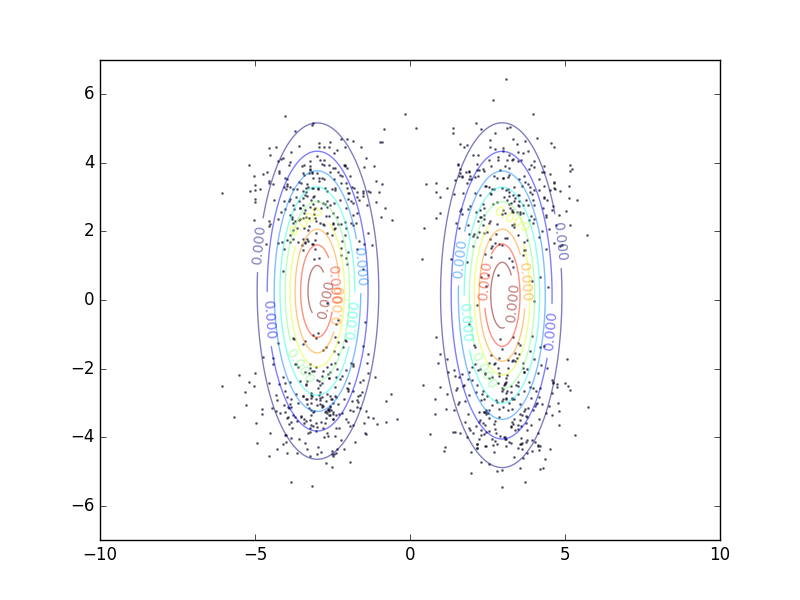}
\caption{GenBGM $(1)$}
\end{subfigure}
~
\begin{subfigure}[b]{0.22\textwidth}
\includegraphics[width=\textwidth]{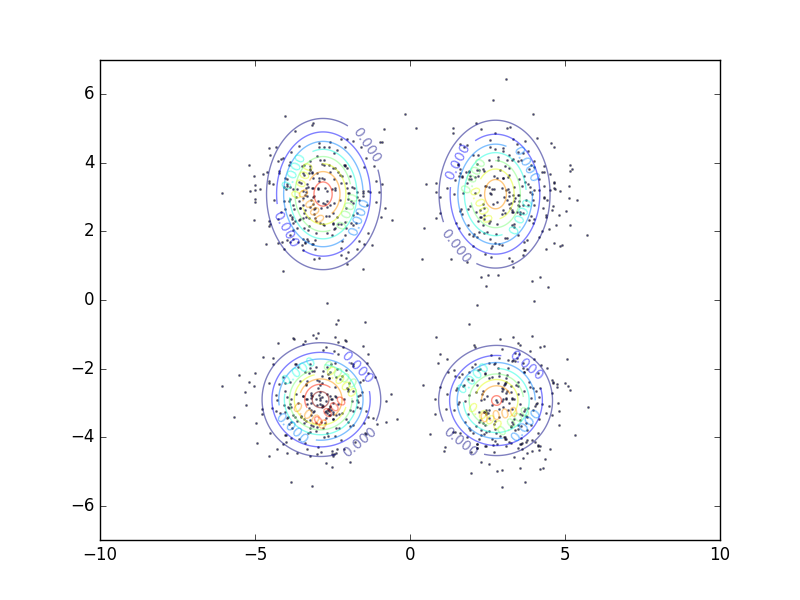}
\caption{GenBGM $(2)$}
\end{subfigure}

\begin{subfigure}[b]{0.22\textwidth}
\includegraphics[width=\textwidth]{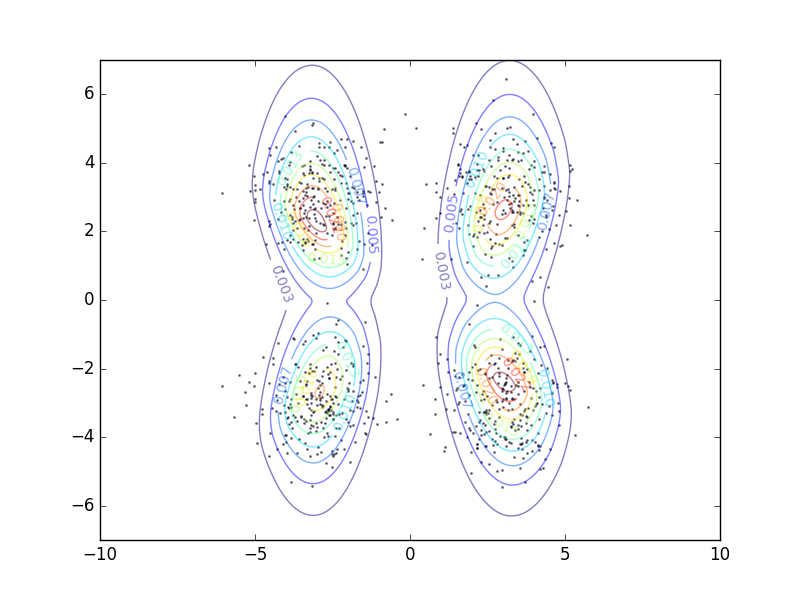}
\caption{DiscBGM-NCE $(1)$}
\end{subfigure}
~
\begin{subfigure}[b]{0.22\textwidth}
\includegraphics[width=\textwidth]{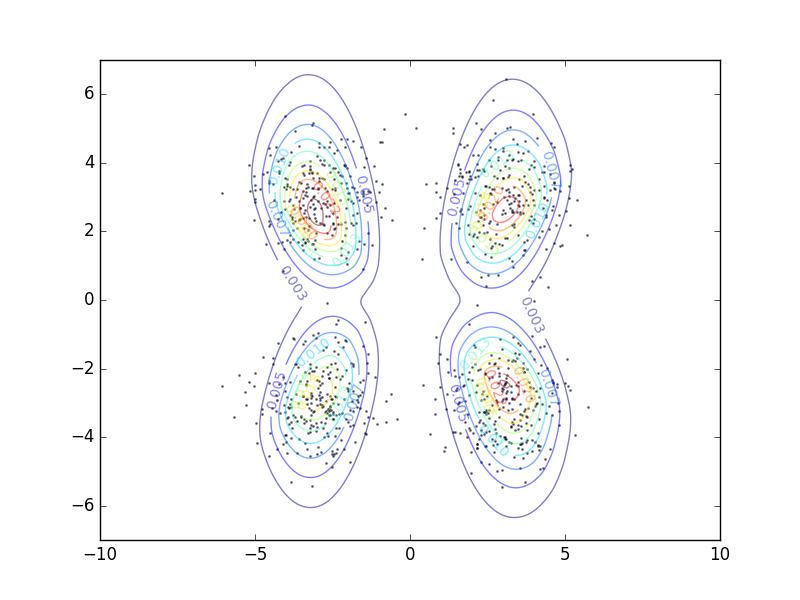}
\caption{DiscBGM-NCE $(2)$}
\end{subfigure}
~
\begin{subfigure}[b]{0.22\textwidth}
\includegraphics[width=\textwidth]{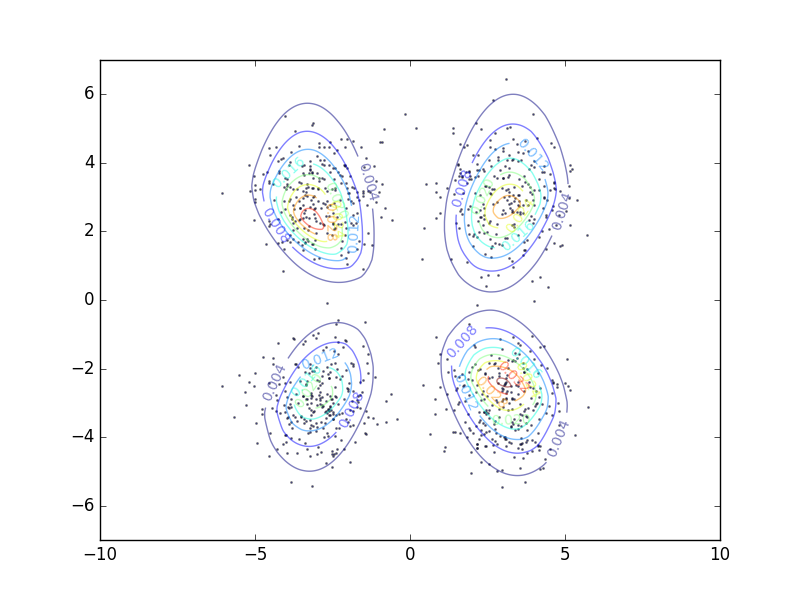}
\caption{DiscBGM-HD $(1)$}
\end{subfigure}
~
\begin{subfigure}[b]{0.22\textwidth}
\includegraphics[width=\textwidth]{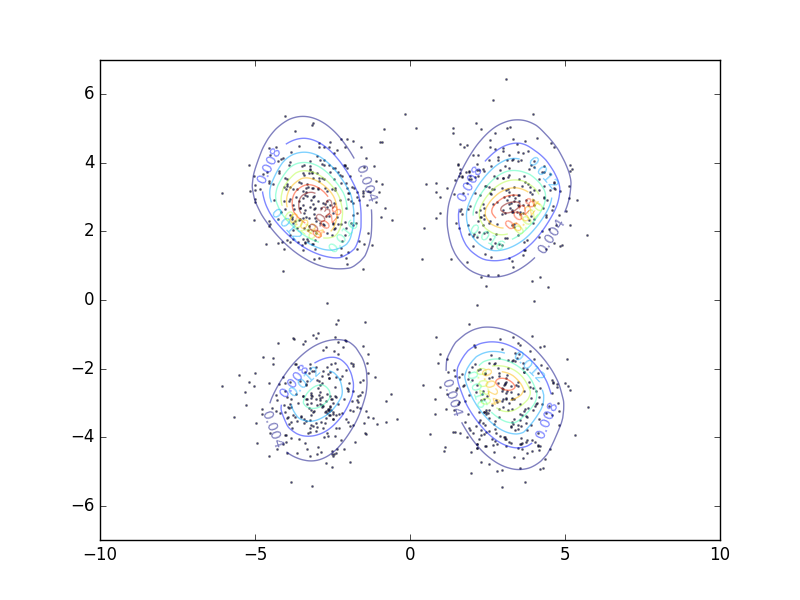}
\caption{DiscBGM-HD $(2)$}
\end{subfigure}

\caption{Multiplicative boosting algorithms such as GenBGM (c-d) and DiscBGM with negative cross-entropy (e-f) and Hellinger distance (g-h) outperform additive boosting (a-b) in correcting for model misspecification. Numbers in parenthesis indicate boosting round $t$.}\label{fig:mog_density_estimation}
\end{figure*}

\begin{table}[t]
  \centering
  \caption{Average test NLL for mixture of Gaussians.}\label{tab:mog_density_estimation}
\begin{tabular}{lc}
 \Tstrut Model  &  NLL (in nats, with std. error) \\ \hline
 \Tstrut Base model  &  $4.69 \pm 0.01$  \\
 \Tstrut Add model  &  $4.64 \pm 0.02$  \\
 \Tstrut GenBGM  &   $4.58 \pm 0.10$  \\ 
 \Tstrut DiscBGM-NCE &   $4.42 \pm 0.01$ \\ 
 \Tstrut DiscBGM-HD    &  $\mathbf{4.35 \pm 0.01}$ \\
\end{tabular}
\end{table}

\subsection{Multiplicative vs. additive boosting}
A common pitfall with learning parameteric generative models is model misspecification with respect to the true underlying data distribution. For a quantitative and qualitative understanding of the behavior of additive and multiplicative boosting, we begin by considering a synthetic setting for density estimation on a mixture of Gaussians.

\paragraph{Density estimation on synthetic dataset.}\label{sec:mog}
The true data distribution is a equi-weighted mixture of four Gaussians centered symmetrically around the origin, each having an identity covariance matrix. The contours of the underlying density are shown in Figure~\ref{fig:mog_setup_target}. We observe $1,000$ training samples drawn independently from the data distribution (shown as black dots in Figure~\ref{fig:mog_density_estimation}), and the task is to learn this distribution. The test set contains $1,000$ samples from the same distribution. We repeat 
the process $10$ times for statistical significance.

As a base (misspecified) model, we fit a mixture of two Gaussians to the data; the contours for an example instance are shown in Figure~\ref{fig:mog_setup_base}. We  compare multiplicative and additive boosting, each run for $T=2$ rounds.
For additive boosting (Add), we extend the algorithm proposed by \citeauthor{rosset2002boosting}~\shortcite{rosset2002boosting} setting $\hat{\alpha}_0$ to unity and doing a line search over $\hat{\alpha}_1, \hat{\alpha}_2 \in [0, 1]$. 
For Add and GenBGM, the intermediate models are mixtures of two Gaussians as well. 
The classifiers for DiscBGM are multi-layer perceptrons with two hidden layers of 100 units each and ReLU activations, trained to maximize $f$-divergences corresponding to the negative cross-entropy (NCE) and Hellinger distance (HD) using the Adam optimizer~\cite{kingma-iclr2014}. 

The test negative log-likelihood (NLL) estimates are listed in Table~\ref{tab:mog_density_estimation}. Qualitatively, the contour plots for the estimated densities after every boosting round on a sample instance are shown in Figure~\ref{fig:mog_density_estimation}. Multiplicative boosting algorithms outperform additive boosting in correcting for model misspecification. GenBGM initially leans towards maximizing coverage, whereas both versions of DiscBGM are relatively more conservative in assigning high densities to data points away from the modes.

\paragraph{Heuristic model weighting strategies.}
The multiplicative boosting algorithms require as hyperparameters the number of rounds of boosting and weights assigned to the intermediate models. For any practical setting, these hyperparameters are specific to the dataset and task under consideration and should be set based on cross-validation. While automatically setting model weights is an important direction for future work, we propose some heuristic weighting strategies.  Specifically, the \textit{unity} heuristic assigns a weight of $1$ to every model in the ensemble, the \textit{uniform} heuristic assigns a weight of $1/(T+1)$ to every model, and the \textit{decay} heuristic assigns as a weight of $1/2^t$ to the $t^{th}$ model in the ensemble.

\begin{figure}[t]
\centering
\begin{subfigure}[b]{0.48\columnwidth}
\includegraphics[width=\textwidth]{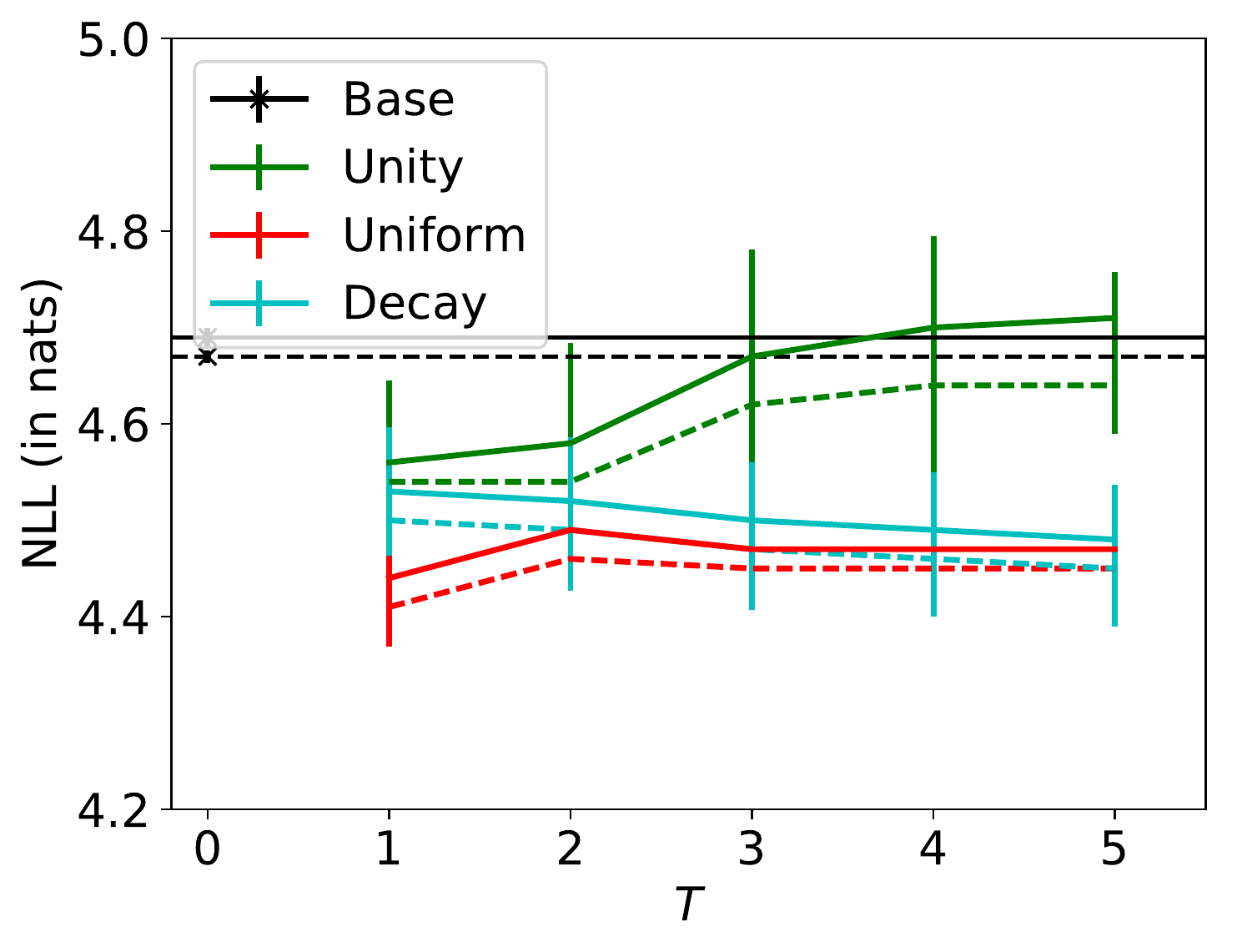}
\caption{GenBGM}
\end{subfigure}
~
\begin{subfigure}[b]{0.48\columnwidth}
\includegraphics[width=\textwidth]{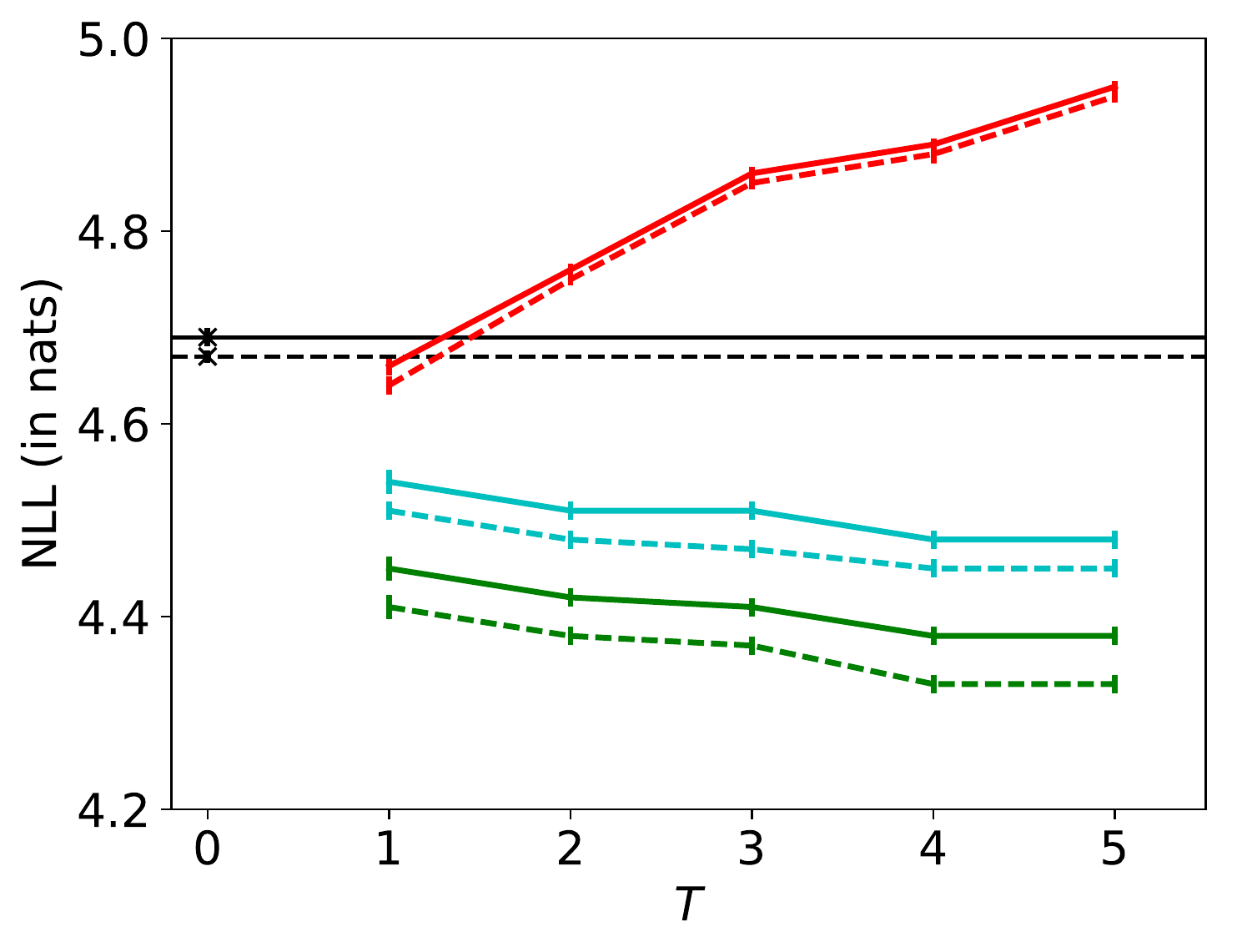}
\caption{DiscBGM-NCE}
\end{subfigure}
\caption{Train (dashed curves) and test (bold curves) NLL (in nats) for weighting heuristics on mixture of Gaussians. $T$ is the number of rounds of boosting. The base model is shown as a black cross at $T=0$.}
\label{fig:heuristics}
\end{figure}

In Figure~\ref{fig:heuristics}, we observe that the performance of the algorithms is sensitive to the weighting strategies. In particular, DiscBGM produces worse estimates as $T$ increases for the ``uniform" (red) strategy. 
The performance of GenBGM also degrades slightly with increasing $T$ for the ``unity'' (green) strategy. 
Notably, the ``decay'' (cyan) strategy achieves stable performance for both the algorithms.
Intuitively, this heuristic follows the rationale of reducing the step size in gradient based stochastic optimization algorithms, and we expect this strategy to work better even in other settings. However, this strategy could potentially result in slower convergence as opposed to the unity strategy.

\begin{table*}[t]
\caption{Experimental results for density estimation. Negative log-likelihoods reported in nats. Lower is better with best performing models in \textbf{bold}. Overall, multiplicative boosting outperforms additive boosting and baseline models specified as Mixture of Bernoullis (MoB, \textbf{middle} columns) and Sum Product Networks (SPN, \textbf{right} columns).}
\label{tab:binary_density_estimation}
\centering\resizebox{0.95\textwidth}{!}{ %
\begin{tabular}{lc|cccc|cccc}
Dataset  & $\#$vars  & MoB Base  & Add  & GenBGM  & DiscBGM  & SPN Base  & Add & GenBGM  & DiscBGM\tabularnewline
\hline 
Accidents  & $111$ & $42.23$	&$43.19$ &	$41.43$&	\textbf{34.51} &31.08&	29.92	&29.55	&\textbf{28.09} \tabularnewline
Retail  & $135$ & 11.27	&12.24	&11.20	&\textbf{10.91} &14.94&	11.27&	11.21	&\textbf{10.88} \tabularnewline
Pumsbstar  & $163$  &55.67&	55.91	&50.66&	\textbf{34.93}& 26.70&	25.00	&25.00&	\textbf{23.69}\tabularnewline
DNA  & $180$  &99.42	&100.37	&99.23	&\textbf{98.45} &92.60&	86.93	&87.79	&\textbf{86.63}   \tabularnewline
Kosarek  & $190$ &11.72	&12.57	&12.41	&\textbf{11.13} &12.71	&10.97	&10.73	&\textbf{10.67 } \tabularnewline
Ad  & $1556$ &63.13	&63.73	&63.19	&\textbf{54.79} &19.19	&18.12	&18.14	&\textbf{17.82}\tabularnewline
\end{tabular}
} 
\end{table*}
\begin{table*}[ht]
\caption{Experimental results for classification. Prediction accuracy for predicting one variable given the rest. Higher is better with best performing models in \textbf{bold}. Multiplicative boosting again outperforms additive boosting and baseline models specified as Mixture of Bernoullis (MoB, \textbf{middle} columns) and Sum Product Networks (SPN, \textbf{right} columns).}
\label{tab:binary_classification}
\centering\resizebox{0.95\textwidth}{!}{ 
\begin{tabular}{lc|cccc|cccc}
Dataset  & $\#$test  & MoB Base  & Add  & GenBGM  & DiscBGM  & SPN Base  & Add & GenBGM  & DiscBGM\tabularnewline
\hline 
Accidents  & 283,161  &0.8395&	0.8393	&0.8473&	\textbf{0.9043}& 0.9258&	0.9266&	0.9298&	\textbf{0.9416}\tabularnewline
Retail  & 595,080 & 0.9776&	0.9776&	0.9776	&\textbf{0.9792} &0.9780	&0.9790	&0.9789	&\textbf{0.9791} \tabularnewline
Pumsb-star  & 399,676 &0.8461&	0.8501	&0.8819	&\textbf{0.9267}& 0.9599	&0.9610	&0.9611	&\textbf{0.9636} \tabularnewline
DNA  & 213,480  &0.7517	&0.7515&	\textbf{0.7531}	&0.7526& 0.7799&	0.7817&	\textbf{0.7828}	&0.7811   \tabularnewline
Kosarek  & 1,268,250 &0.9817	&0.9816&	0.9818	&\textbf{0.9831}& 0.9824&	\textbf{0.9838}&	\textbf{0.9838}&	\textbf{0.9838}  \tabularnewline
Ad  & 763,996& 0.9922&	0.9923	&0.9818	&\textbf{0.9927}& \textbf{0.9982}&	0.9981&	\textbf{0.9982}&	\textbf{0.9982} \tabularnewline
\end{tabular}
} 
\end{table*}

\paragraph{Density estimation on benchmark datasets.} We now evaluate the performance of additive and multiplicative boosting for density estimation on real-world benchmark datasets~\cite{van2012markov}. We consider two generative model families: mixture of Bernoullis (MoB) and sum-product networks~\cite{poon2011sum}. While our results for multiplicative boosting with sum-product networks (SPN) are competitive with the state-of-the-art, the goal of these experiments is to perform a robust comparison of boosting algorithms as well as demonstrate their applicability to various model families.

We set $T=2$ rounds for additive boosting and GenBGM. Since DiscBGM requires samples from the model density at every round, we set $T=1$ to ensure computational fairness such that the samples can be obtained efficiently from the base model sidestepping running expensive Markov chains. Model weights are chosen based on cross-validation. 
The results on density estimation are reported in Table~\ref{tab:binary_density_estimation}.
Since multiplicative boosting estimates are unnormalized, we use importance sampling to estimate the partition function. 

When the base model is MoB, the Add model underperforms and is often worse than even the baseline model for the best performing validated non-zero model weights. GenBGM consistently outperforms Add and improves over the baseline model in a most cases (4/6 datasets). DiscBGM performs the best and convincingly outperforms the baseline, Add, and  GenBGM on all datasets. For results on SPNs, the boosted models all outperform the baseline. GenBGM again edges out Add models (4/6 datasets), whereas DiscBGM models outperform all other models on all datasets.  These results demonstrate the usefulness of boosted expressive model families, especially the DiscBGM approach, which performs the best, while GenBGM is preferable to Add.

\subsection{Applications of generative models}

\paragraph{Classification.} Here, we evaluate the performance of boosting algorithms for classification. Since the datasets above do not have any explicit labels, we choose one of the dimensions to be the label (say $y$). Letting $\mathbf{x}_{\bar{y}}$ denote the remaining dimensions, we can obtain a prediction for $y$ as, 
\[p(y=1\vert \mathbf{x}_{\bar{y}})= \frac{p(y=1, \mathbf{x}_{\bar{y}})}{p(y=1, \mathbf{x}_{\bar{y}}) + p(y=0, \mathbf{x}_{\bar{y}})}\]
which is efficient to compute even for unnormalized models. We repeat the above procedure for all the variables predicting one variable at a time using the values assigned to the remaining variables. The results are reported in Table~\ref{tab:binary_classification}. When the base model is a MoB, we observe that the Add approach could often be worse than the base model whereas GenBGM performs slightly better than the baseline (4/6 datasets). The DiscBGM approach consistently performs well, and is only outperformed by GenBGM for two datasets for MoB. When SPNs are used instead, both Add and GenBGM improve upon the baseline model while DiscBGM again is the best performing model on all but one dataset.

\begin{figure*}[t]
\centering
\begin{subfigure}[b]{0.24\textwidth}
\centering
\includegraphics[width=\textwidth]{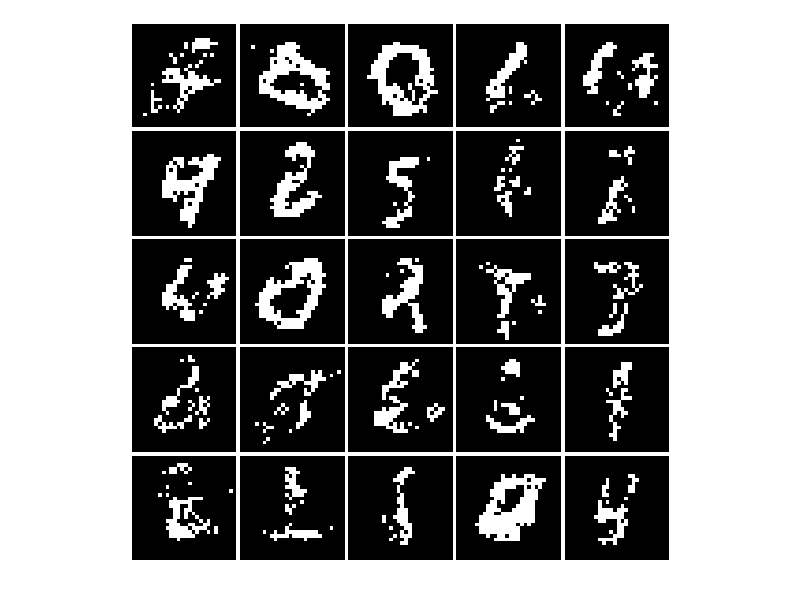}
\caption{Base VAE \\
(200-100)}\label{fig:mnist_sampling_base}
\end{subfigure}
\begin{subfigure}[b]{0.24\textwidth}
\centering
\includegraphics[width=\textwidth]{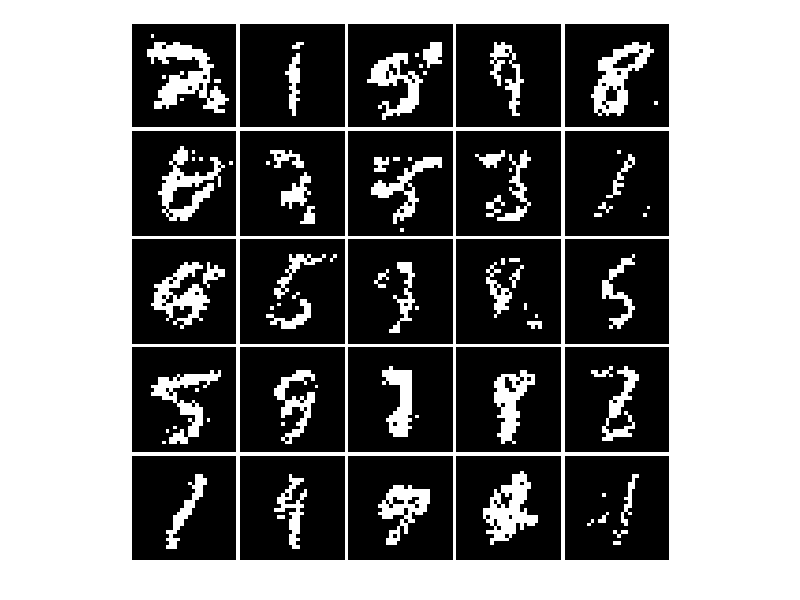}
\caption{ Base + depth \\
(200-100-100)
}\label{fig:mnist_sampling_base_depth}
\end{subfigure}
\begin{subfigure}[b]{0.24\textwidth}
\centering
\includegraphics[width=\textwidth]{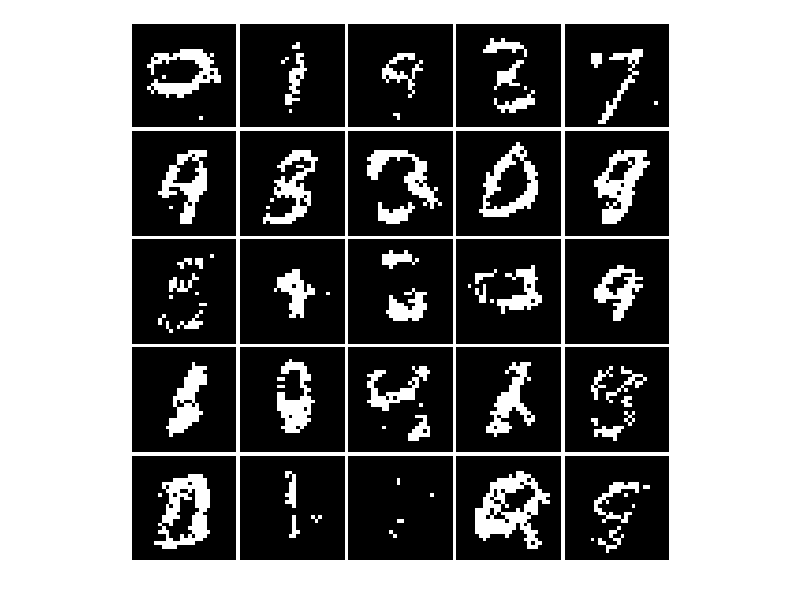}
\caption{Base + width \\ (300-100)}\label{fig:mnist_sampling_base_width}
\end{subfigure}
\begin{subfigure}[b]{0.24\textwidth}
\centering
\includegraphics[width=\textwidth]{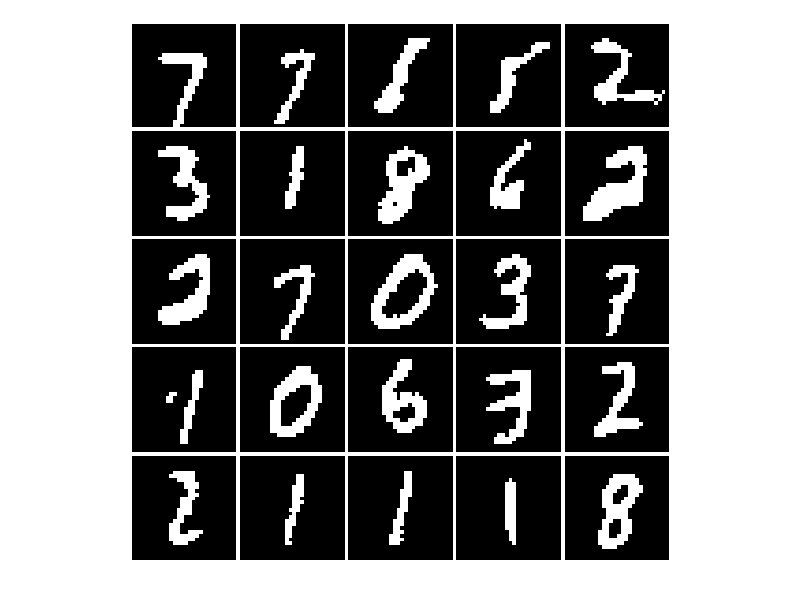}
\caption{GenDiscBGM \\
(100-50)}\label{fig:mnist_sampling_bgm}
\end{subfigure}

\caption{The boosted model (d) demonstrates how ensembles of weak learners can generate sharper samples, compared to naively increasing model capacity (a-c). Note that we show samples of binarized digits and not mean values for the pixels. VAE hidden layer architecture given in parenthesis. }\label{fig:mnist_sampling}
\end{figure*}

\paragraph{Sample generation.} 
We compare boosting algorithms based on their ability to generate image samples for the binarized MNIST dataset of handwritten digits~\cite{lecun2010mnist}. We use variational autoencoders (VAE) as the base model~\cite{kingma-iclr2014}. While any sufficiently expressive VAE can generate impressive examples, we design the experiment to evaluate the model complexity approximated as the number of learnable parameters.

Ancestral samples obtained by the baseline VAE model are shown in Figure~\ref{fig:mnist_sampling_base}. We use the evidence lower bound (ELBO) as a proxy for approximately evaluating the marginal log-likelihood during learning. The conventional approach to improving the performance of a latent variable model is to increase its representational capacity by adding hidden layers (Base + depth) or increasing the number of hidden units in the existing layers (Base + width). These lead to a marginal improvement in sample quality as seen in Figure~\ref{fig:mnist_sampling_base_depth} and Figure~\ref{fig:mnist_sampling_base_width}. 

In contrast,  boosting makes steady improvements in sample quality. We start with a VAE with much fewer parameters and generate samples using a hybrid boosting GenDiscBGM sequence VAE$\rightarrow$CNN$\rightarrow$VAE (Figure~\ref{fig:mnist_sampling_bgm}) . The discriminator used is a convolutional neural network (CNN)~\cite{lecun1995convolutional} trained to maximize the negative cross-entropy. We then generate samples using independent Markov chain Monte Carlo (MCMC) runs.
The boosted sequences generate sharper samples than all baselines in spite of having similar model capacity. 

\section{Discussion and related work}\label{sec:rel}
In this work, we revisited boosting, a class of meta-algorithms developed in response to a seminal
question: \textit{Can a set of weak learners create a single strong learner?}
Boosting has offered interesting theoretical insights into the fundamental limits of supervised learning and led to the development of algorithms that work well in practice ~\cite{schapire1990strength,freund1999short,friedman2002stochastic,caruana2006empirical}. 
Our work provides a foundational framework for unsupervised boosting with connections to prior work discussed below.

\paragraph{Sum-of-experts.}
\citeauthor{rosset2002boosting}~\shortcite{rosset2002boosting} proposed an algorithm for density estimation using Bayesian networks similar to gradient boosting.
These models are normalized and easy to sample, but are generally outperformed by multiplicative formulations for correcting for model misspecification, as we show in this work. Similar additive approaches have been used for improving approximate posteriors for specific algorithms for variational inference~\cite{guo2016boosting,miller2016variational} and generative adversarial networks~\cite{tolstikhin2017adagan}. For a survey on variations of additive ensembling for unsupervised settings, refer to the survey by~\citeauthor{bourel2012aggregating}~\shortcite{bourel2012aggregating}. 

\paragraph{Product-of-experts.}
Our multiplicative boosting formulation can be interpreted as a product-of-experts approach, which was initially proposed for feature learning in energy based models such as Boltzmann machines. For example, the hidden units in a restricted Boltzmann machine can be interpreted as weak learners performing MLE.
If the number of weak learners is fixed, they can be efficiently updated in parallel but there is a risk of learning redundant features
~\cite{hinton1999products,hinton2002training}. 
Weak learners can also be added incrementally based on the learner's ability to distinguish observed data and model-generated data~\cite{welling2002self}. \citeauthor{tu2007learning}~\shortcite{tu2007learning} generalized the latter to boost arbitrary probabilistic models; their algorithm is a special case of DiscBGM  with all $\alpha$'s set to 1 and the discriminator itself a boosted classifier. DiscBGM additionally accounts for imperfections in learning classifiers through flexible model weights. Further, it can include any classifier trained to maximize any $f$-divergence. 

Related techniques such as noise-contrastive estimation, ratio
matching, and score matching methods can be cast as minimization of Bregman divergences, akin to DiscBGM with unit model weights~\cite{gutmann2012bregman}. A non-parametric algorithm similar to GenBGM was proposed by \citeauthor{di2004boosting}~\shortcite{di2004boosting} where an ensemble of weighted kernel density estimates are learned to approximate the data distribution. In contrast, our framework allows for both parametric and non-parametric learners and uses a different scheme for reweighting data points than proposed in the above work.

\paragraph{Unsupervised-as-supervised learning.}
The use of density ratios learned by a binary classifier for estimation was first proposed by~\citeauthor{friedman2001elements}~\shortcite{friedman2001elements} and has been subsequently applied elsewhere, notably for parameter estimation using noise-contrastive estimation~\cite{gutmann2010noise} and sample generation in generative adversarial networks (GAN)~\cite{goodfellow2014generative}. While GANs consist of a discriminator distinguishing real data from model generated data similar to DiscBGM for a suitable $f$-divergence, they differ in the learning objective for the generator~\cite{nowozin2016f}. The generator of a GAN performs an adversarial minimization of the same objective the discriminator maximizes, whereas DiscBGM uses the likelihood estimate of the base generator (learned using MLE) and the density ratios derived from the discriminator(s) to estimate the model density for the ensemble.

\paragraph{Limitations and future work.} In the multiplicative boosting framework, the model density needs to be specified only up to a normalization constant at any given round of boosting. Additionally, while many applications of generative modeling such as feature learning and classification can sidestep computing the partition function, if needed it can be estimated using techniques such as Annealed Importance Sampling~\cite{neal2001annealed}. Similarly, Markov chain Monte Carlo methods can be used to generate samples. The lack of implicit normalization can however be limiting for applications requiring fast log-likelihood evaluation and sampling. 

In order to sidestep this issue, a promising direction for future work is to consider boosting of normalizing flow models~\cite{dinh2014nice,dinh2016density,grover2017flow}. These models specify an invertible multiplicative transformation from one  distribution to another using the change-of-variables formula such that the resulting distribution is self-normalized and efficient ancestral sampling is possible. 
The GenBGM algorithm can be adapted to normalizing flow models whereby every transformation is interpreted as a weak learner. The parameters for every transformation can be trained greedily after suitable reweighting resulting in a self-normalized boosted generative model.

\section{Conclusion}\label{sec:conc}
We presented a general-purpose framework for boosting generative models by explicit factorization of the model likelihood as a product of simpler intermediate model densities. These intermediate models are learned greedily using discriminative or generative approaches, gradually increasing the overall model's capacity. 
We demonstrated the effectiveness of these models over baseline models and additive boosting for the tasks of density estimation, classification, and sample generation. 
Extensions to semi-supervised learning~\cite{kingma2014semi} and structured prediction~\cite{sohn2015learning} are
exciting directions for future work.

\section*{Acknowledgements}
We are thankful to Neal Jean, Daniel Levy, and Russell Stewart for helpful critique. This research was supported by a Microsoft Research PhD fellowship in machine learning for the first author, NSF grants $\#1651565$, $\#1522054$, $\#1733686$, a Future of Life Institute grant, and Intel.
\fontsize{9pt}{10pt} \selectfont
\bibliographystyle{aaai}
\bibliography{refs}

\newpage
\onecolumn
\section*{Appendices}
\begin{appendices}
\section{Proofs of theoretical results}

\subsection{Theorem~\ref{thm:add_KL_red}}\label{proof:add_KL_red}
\begin{proof}

\text{The reduction in KL-divergence can be simplified as:}\\
\begin{align*}
\delta^t_{KL}(h_t, \hat{\alpha}_t) &= \mathbb{E}_P\left[\log \frac{p}{q_{t-1}}\right] - \mathbb{E}_P\left[\log \frac{p}{q_t}\right] \\
&= \mathbb{E}_P\left[\log \frac{q_t}{q_{t-1}} \right] \\
&= \mathbb{E}_P\left[\log \left[ (1-\hat{\alpha}_t) + \hat{\alpha}_t \frac{h_t}{q_{t-1}}\right]\right].&
\end{align*}

\noindent We first derive the \textbf{sufficient condition} by lower bounding 
$\delta^t_{KL}(h_t, \hat{\alpha}_t)$.
\begin{align*}
\delta^t_{KL}(h_t, \hat{\alpha}_t) &=\mathbb{E}_P\left[\log \left[ (1-\hat{\alpha}_t) + \hat{\alpha}_t \frac{h_t}{q_{t-1}}\right]\right] \\
&\geq\mathbb{E}_P\left[(1-\hat{\alpha}_t) \log 1 + \hat{\alpha}_t \log \frac{h_t}{q_{t-1}}\right] &\text{(Arithmetic Mean} \geq \text{Geometric Mean)}\\
&= \hat{\alpha}_t \mathbb{E}_P\left[\log \frac{h_t}{q_{t-1}}\right]. &\text{(Linearity of expectation)}
\end{align*}
If the lower bound is non-negative, then so is $\delta^t_{KL}(h_t, \hat{\alpha}_t)$. Hence:
\begin{align*}
\mathbb{E}_P\left[\log \frac{h_t}{q_{t-1}}\right] &\geq 0 &
\end{align*}
which is the stated sufficient condition.
\\

\noindent For the \textbf{necessary condition} to hold, we know that:
\begin{align*}
0 &\leq \delta^t_{KL}(h_t, \hat{\alpha}_t) \\
&=  \mathbb{E}_P\left[\log \left[ (1-\hat{\alpha}_t) + \hat{\alpha}_t \frac{h_t}{q_{t-1}}\right]\right] \\
&\leq \log \mathbb{E}_P\left[ (1-\hat{\alpha}_t) + \hat{\alpha}_t \frac{h_t}{q_{t-1}}\right] & \text{(Jensen's inequality)}\\
&= \log \left [ (1-\hat{\alpha}_t) + \hat{\alpha}_t \mathbb{E}_P\left[  \frac{h_t}{q_{t-1}}\right]\right] & \text{(Linearity of expectation)}
\end{align*}
Taking exponential on both sides, we get:
\begin{align*}
(1-\hat{\alpha}_t) + \hat{\alpha}_t \mathbb{E}_P\left[  \frac{h_t}{q_{t-1}}\right] &\geq 1\nonumber \\
\mathbb{E}_P\left[  \frac{h_t}{q_{t-1}}\right] &\geq 1 &
\end{align*}
which is the stated necessary condition.
\end{proof}

\subsection{Theorem~\ref{thm:KL_red}}\label{proof:KL_red}
\begin{proof}
We first derive the \textbf{sufficient condition}.
\begin{align}\label{eq:greedy_objective}
\delta^t_{KL}(h_t, \alpha_t) &= \int p \log q_t \,\mathrm{d}\mathbf{x} - \int p \log q_{t-1} \,\mathrm{d}\mathbf{x}\nonumber \\\nonumber 
&= \int p \log \frac{h_t^{\alpha_t} \cdot q_{t-1}}{Z_t} - \int p \log q_{t-1} &\text{(using Eq.~\eqref{eq:q_t})}\\ 
&= \alpha_t \cdot \mathbb{E}_P[\log h_t] - \log \mathbb{E}_{Q_{t-1}}[h_t^{\alpha_t}]\\\nonumber 
&\geq \alpha_t \cdot \mathbb{E}_P[\log h_t] - \log \mathbb{E}_{Q_{t-1}}[h_t]^{\alpha_t} & \text{(Jensen's inequality)}\\\nonumber 
&= \alpha_t \cdot \big[\mathbb{E}_P[\log h_t] - \log \mathbb{E}_{Q_{t-1}}[h_t]\big] \\\nonumber 
&\geq 0. & \text{(by assumption)}
\end{align}
Note that if $\alpha_t=1$, the sufficient condition is also necessary. 
\\

\noindent For the \textbf{necessary condition} to hold, we know that:
\begin{align*}
0 \leq \delta^t_{KL}(h_t, \alpha_t) &= \alpha_t \cdot \mathbb{E}_P[\log h_t]  - \log \mathbb{E}_{Q_{t-1}}[h_t^{\alpha_t}]\\
&\leq \alpha_t \cdot \mathbb{E}_P[\log h_t]  - \mathbb{E}_{Q_{t-1}}[\log h_t^{\alpha_t}]& \text{(Jensen's inequality)}\\
&=\alpha_t \cdot [\mathbb{E}_P[\log h_t] - \mathbb{E}_{Q_{t-1}} [\log h_t]] & \text{(Linearity of expectation)}\\
&\leq \mathbb{E}_P[\log h_t] - \mathbb{E}_{Q_{t-1}} [\log h_t]. & (\text{since } \alpha_t > 0)
\end{align*}
\end{proof}

\subsection{Proposition~\ref{thm:genbgm_reweight}}\label{proof:genbgm_reweight}
\begin{proof}
By assumption, we can optimize Eq.~\eqref{eq:genbgm_obj} to get:
\begin{align*}
h_t &\propto \left(\frac{p}{q_{t-1}}\right)^{\beta_t}&.
\end{align*}

\noindent Substituting for $h_t$ in the multiplicative boosting formulation in Eq.~\eqref{eq:q_t},:
\begin{align*}
q_t &\propto \frac{q_{t-1} \cdot h_t}{Z_{q_t}}\\
&\propto q_{t-1} \cdot \left(\frac{p}{q_{t-1}}\right)^{\beta_t}\\
&= \frac{ p^{\beta_t} \cdot  q_{t-1}^{1-\beta_t} }{Z_{q_t}}&
\end{align*}
where the partition function $Z_{q_t} = \int p^{\beta_t} \cdot  q_{t-1}^{1-\beta_t} $.
\\

\noindent In order to prove the inequality, we first obtain a lower bound on the log-partition function, $Z_{q_t}$. For any given point, we have:
\begin{align*}
 p^{\beta_t} \cdot  q_{t-1}^{1-\beta_t} &\leq \beta_t p +  (1-\beta_t) q_{t-1}. & \text{(Arithmetic Mean $\geq$ Geometric Mean)}
\end{align*}
Integrating over all points in the domain, we get: 
\begin{align}\label{eq:lower_bound_Z}
\log Z_{q_t} &\leq \log \left[\beta Z_p + (1-\beta) Z_{q_{t-1}} \right] \nonumber \\
&= 0 &
\end{align}
where we have used the fact that $p$ and $q_{t-1}$ are normalized densities. 
\\

\noindent Now, consider the following quantity:
\begin{align*}
D_{KL}(P \Vert Q_t) &= \mathbb{E}_P \left[\log \frac{p}{q_t}\right]  \\
&= \mathbb{E}_P \left[\log \frac{p}{\frac{p^{\beta_t} \cdot  q_{t-1}^{1-\beta_t}}{Z_{q_t}}}\right]  \\
&= (1-\beta_t) \mathbb{E}_P \left[\log \frac{p}{q_{t-1}}\right] + \log Z_{q_t}\\
&\leq (1-\beta_t) \mathbb{E}_P \left[\log \frac{p}{q_{t-1}}\right] & \text{(using Eq.~\eqref{eq:lower_bound_Z})} \\
&\leq \mathbb{E}_P \left[\log \frac{p}{q_{t-1}}\right] & (\text{since } \beta_t \geq 0) \\
&= D_{KL}(P \Vert Q_{t-1}).
\end{align*}
\end{proof}

\subsection{Proposition~\ref{thm:f_optimal}}\label{proof:f_optimal}
\begin{proof}
By the $f$-optimality assumption, we know that:
\begin{align*}
r_t &= f'\left(\frac{p}{q_{t-1}}\right).&
\end{align*}
Hence, $h_t = \frac{p}{q_{t-1}}$. From Eq.~\eqref{eq:q_t}, we get:
\begin{align*}
q_t &= q_{t-1} \cdot h_t^{\alpha_t} = p&
\end{align*}
finishing the proof.
\end{proof}

\subsection{Corollary~\ref{thm:bayes_optimal}}\label{proof:bayes_optimal}
\begin{proof}
Let $u_t$ denote the joint distribution over $(\mathbf{x}, y)$ at round $t$. We will prove a slightly more general result where we have $m$ positive training examples sampled from $p$ and the $k$ negative training examples sampled from $q_{t-1}$.\footnote{In the statement for Corollary~\ref{thm:bayes_optimal}, the classes are assumed to be balanced for simplicity \textit{i.e.}, $m=k$.} Hence, we can express the conditional and prior densities as:
\begin{align}
p &= u(\mathbf{x} \vert y=1) \label{eq:binary_class_1} \\
q_{t-1} &= u(\mathbf{x} \vert y=0) \label{eq:binary_class_2}\\
u(y=1) &= \frac{m}{m+k} \label{eq:prior_class_1}\\
u(y=0) &= \frac{k}{m+k} \label{eq:prior_class_2}.&
\end{align}
The Bayes optimal density $c_t$ can be expressed as:
\begin{align}
c_t &= u(y=1 \vert \mathbf{x}) 
\nonumber\\
&= u(\mathbf{x} \vert y=1)  u(y=1) / u(\mathbf{x}) \label{eq:bayes_d_1}.&
\end{align}
Similarly, we have:
\begin{align}
1-c_t &= u(\mathbf{x} \vert y=0) u(y=0) / u(\mathbf{x})
\label{eq:bayes_d_2}.&
\end{align}
From Eqs.~(\ref{eq:binary_class_1}-\ref{eq:prior_class_2}, \ref{eq:bayes_d_1}-\ref{eq:bayes_d_2}), we have:
\begin{align*}
h_t &= \gamma \cdot \frac{c_t}{1-c_t} = \frac{p}{q_{t-1}}.&
\end{align*}
where $\gamma = \frac{k}{m}$. 
\\

\noindent Finally from Eq.~\eqref{eq:q_t}, we get:
\begin{align*}
q_t &= q_{t-1} \cdot h_t^{\alpha_t} = p &
\end{align*}
finishing the proof.
\end{proof}

In Corollary~\ref{thm:adversarial_bayes_optimal} below, we present an additional theoretical result below that derives the optimal model weight, $\alpha_t$ for an \textit{adversarial Bayes optimal classifier}. 

\subsection{Corollary~\ref{thm:adversarial_bayes_optimal}}
\begin{corollary}~\label{thm:adversarial_bayes_optimal} [to Corollary~\ref{thm:bayes_optimal}]
Define an adversarial Bayes optimal classifier $c_t'$ as one that assigns the density $c_t' = 1 - c_t$ where $c_t$ is the Bayes optimal classifier.
For an adversarial Bayes optimal classifier $c_t'$, $\delta^t_{KL}$ attains a maxima of zero when $\alpha_t=0$. 
\end{corollary}

\begin{proof}
For an adversarial Bayes optimal classifier,
\begin{align}
c_t' &= u(\mathbf{x} \vert y=0) u(y=0) / u(\mathbf{x})\label{eq:adv_bayes_d_1}\\
1-c_t' &= u(\mathbf{x} \vert y=1) u(y=1) / u(\mathbf{x})\label{eq:adv_bayes_d_2}.&
\end{align}
From Eqs.~(\ref{eq:binary_class_1}-\ref{eq:prior_class_2},   \ref{eq:adv_bayes_d_1}-\ref{eq:adv_bayes_d_2}), we have:
\begin{align*}
h_t &= \gamma \cdot \frac{c_t'}{1-c_t'} = \frac{q_{t-1}}{p}.&
\end{align*}
Substituting the above intermediate model in Eq.~\eqref{eq:greedy_objective},
\begin{align*}
\delta^t_{KL}(h_t, \alpha_t) &= \alpha_t \cdot \mathbb{E}_P\left[\log \frac{q_{t-1}}{p}\right] - \log \mathbb{E}_{Q_{t-1}}\left[\frac{q_{t-1}}{p}\right]^{\alpha_t}\\
&\leq \alpha_t \cdot \mathbb{E}_P\left[\log \frac{q_{t-1}}{p}\right] - \mathbb{E}_{Q_{t-1}}\left[\alpha_t \cdot \log \frac{q_{t-1}}{p}\right] & \text{(Jensen's inequality)}\\
&= \alpha_t \cdot \left [\mathbb{E}_P\left[\log \frac{q_{t-1}}{p}\right] - \mathbb{E}_{Q_{t-1}}\left[\log \frac{q_{t-1}}{p}\right]\right] &\text{(Linearity of expectation)}\\
&= -\alpha_t \left[D_{KL}(P \parallel Q_{t-1}) + D_{KL}(Q_{t-1} \parallel P) \right]\\
&\leq 0 & (D_{KL}\text{ is non-negative)}.\\
\end{align*}
By inspection, the equality holds when $\alpha_t=0$ finishing the proof. 
\end{proof}

\section{Additional implementation details}~\label{app:exp}

\subsection{Density estimation on synthetic dataset}
\paragraph{Model weights.} For DiscBGM, all model weights, $\alpha$'s to unity. The model weights for GenBGM, $\alpha$'s are set uniformly to $1/(T+1)$ and reweighting coefficients, $\beta$'s  are set to unity.

\subsection{Density estimation on benchmark datasets}
\paragraph{Generator learning procedure details.} We use the default open source implementations of mixture of Bernoullis (MoB) and sum-product networks (SPN) as given in \texttt{https://github.com/AmazaspShumik/sklearn-bayes} and \texttt{https://github.com/KalraA/Tachyon} respectively for baseline models.

\paragraph{Discriminator learning procedure details.} The discriminator considered for these experiments is a multilayer perceptron with two hidden layers consisting of $100$ units each and ReLU activations learned using the Adam optimizer~\cite{kingma2014adam} with a learning rate of $1e-4$. The training is for $100$ epochs with a mini-batch size of $100$, and finally the model checkpoint with the best validation error during training is selected to specify the intermediate model to be added to the ensemble.

\paragraph{Model weights.} Model weights for multiplicative boosting algorithms, GenBGM and DiscBGM, are set based on best validation set performance of the heuristic weighting strategies. Partition function is estimated using importance sampling with the baseline model (MoB or SPN) as a proposal and a sample size of $1,000,000$.

\subsection{Sample generation}\label{app:mnist}

\paragraph{VAE architecture and learning procedure details.} Only the last layer in every VAE is stochastic, rest are deterministic. The inference network specifying the posterior contains the same architecture for the hidden layer as the generative network. The prior over the latent variables is standard Gaussian, the hidden layer activations are ReLU, and learning is done using Adam~\cite{kingma2014adam} with a learning rate of $10^{-3}$ and mini-batches of size $100$. 

\paragraph{CNN architecture and learning procedure details.} The CNN contains two convolutional layers and a single full connected layer with $1024$ units. Convolution layers have kernel size $5\times 5$, and $32$ and $64$ output channels, respectively. We apply ReLUs and $2\times 2$ max pooling after each convolution. The net is randomly initialized prior to training, and learning is done using the Adam~\cite{kingma2014adam} optimizer with a learning rate of $10^{-3}$ and mini-batches of size $100$.

\paragraph{Sampling procedure for BGM sequences.}
Samples from the GenDiscBGM are drawn from a Markov chain run using the Metropolis-Hastings algorithm with a discrete, uniformly random proposal and the BGM distribution as the stationary distribution for the chain. Every sample in Figure~\ref{fig:mnist_sampling} (d) is drawn from an independent Markov chain with a burn-in period of $100,000$ samples and a different start seed state. 

\end{appendices}
\end{document}